\documentclass[runningheads]{llncs}
\usepackage{graphicx}
\usepackage{amsmath,amssymb} 
\usepackage{color}
\usepackage{pgf,tikz}
\usepackage{mathrsfs}
\usepackage{tabularx}
\usepackage{enumerate}
\usepackage{booktabs}
\usetikzlibrary{arrows.meta}
\tikzset{ 
        myNode/.style={fill, circle, inner sep = 0, outer sep = 0, minimum size=5pt},
        myTip /.tip = {Latex[length=7pt, width=0pt 4]},
        myEdge/.style={thick}
        }

\spnewtheorem*{open}{Open Question}{\bfseries}{\normalfont}

\newcommand{\RR}{\mathbb{R}}
\newcommand{\QQ}{\mathbb{Q}}
\newcommand{\PP}{\mathbb{P}}
\newcommand{\CC}{\mathbb{C}}

\newcommand{\vect}[1]{{\bf{#1}}}

\newcommand{\ie}{{\em i.e.}}
\newcommand{\eg}{{\em e.g.}}

\usepackage{appendix}
\usepackage{chngcntr}

\makeatletter
\newcommand{\@chapapp}{\relax}%
\makeatother

\begin{document}
\pagestyle{headings}
\mainmatter
\def\ECCV18SubNumber{1983}  

\title{On the Solvability of Viewing Graphs} 




\author{Matthew Trager\inst{1,2}
\and 
Brian Osserman\inst{3}
\and
Jean Ponce\inst{1,2}}

\institute{Inria
\and
\'Ecole Normale Sup\'erieure, CNRS, PSL Research University
\and
UC Davis}

\maketitle

\begin{abstract} A set of fundamental matrices relating pairs of cameras in
some configuration can be represented as edges of a ``viewing graph''. Whether
or not these fundamental matrices are generically sufficient to recover the
global camera configuration depends on the structure of this graph. We study
characterizations of ``solvable'' viewing graphs, and present several new
results that can be applied to determine which pairs of views may be used to
recover all camera parameters. We also discuss strategies for verifying the
solvability of a graph computationally.

\end{abstract}

\section{Introduction}

Multi-view geometry has been studied by photogrammeters since the
1950s~\cite{ManPhot} and by computer vision researchers since the
1980s~\cite{longuet1981computer}. Still, most results to date are concerned
with using multi-view tensors to characterize feature correspondences in
2, 3, or 4 views, and determine the corresponding projection
matrices~\cite{Hartley98,hartley1997lines,LuFau95,Shashua95}. Although
correspondences have also been characterized for arbitrary numbers of views
\cite{faugeras1995geometry,heyden1997algebraic,triggs1995matching}, very little
theoretical work has been devoted to understanding the geometric constraints
imposed on configurations of $n>4$ cameras by these tensors, including
fundamental matrices~\cite{LuFau95}, which are probably by far the most used in
practice.
Apart from a few works such as~\cite{levi2003viewing,rudi2010linear},
understanding how many and which fundamental matrices can be used to recover
globally consistent camera parameters is a largely unexplored problem.

We address this topic in this paper. Following~\cite{levi2003viewing}, we
associate sets of fundamental matrices with edges of a ``viewing graph'', and
present a series of new results for determining whether a graph is
``solvable'', \ie, whether it represents fundamental matrices that determine a
unique camera configuration. We also describe effective strategies for
verifying solvability, and include some computational experiments using these
methods. Our focus here is clearly of a theoretical nature, but understanding
how subsets of fundamental matrices constrain the reconstruction process is
clearly important in practice as well. Moreover, we will point out that one of
our main results (Theorem~\ref{thm:moves}) is constructive, and could
potentially find applications in reconstruction algorithms (\eg, it could be
incorporated in large-scale systems such as~\cite{SSS06}, that incrementally
build up networks of cameras to estimate their parameters).


\newpage

\noindent{\bf Previous work.} The first investigation of viewing graphs and
their solvability can be found in~\cite{levi2003viewing}. In that work, Levi
and Werman characterize all solvable viewing graphs with at most six vertices,
and discuss a few larger solvable examples. Although they provide some useful
necessary conditions (see our Proposition~\ref{prop:levi} and
Example~\ref{ex:levi}), they do not address the problem of solvability in
general. In~\cite{rudi2010linear}, Rudi {\em et al.} also consider viewing
graphs, studying mainly whether a configuration can be recovered from a set of
fundamental matrices using a linear system. They also present some
``composition rules'' for merging solvable graphs into larger ones. Trager
{\em et. al}~\cite{trager2015joint} provide a sufficient condition for
solvability using $2n-3$ fundamental matrices, and point to a possible
connection with ``Laman graphs'' and graph rigidity theory. Indeed, \"Ozyesil
and Singer~\cite{ozyesil2015robust} show that if one uses {\em essential}
matrices instead of fundamental ones then solvability can be characterized in
terms of so-called ``parallel-rigidity'' for graphs. Their analysis however
does not carry over to the more general setting of uncalibrated cameras.
Finally, the viewing graph has also been considered in more practical work:
for example, in~\cite{sinha2010camera,sweeney2015optimizing}, it is
used to enforce triple-wise consistency among fundamental matrices before
estimating camera parameters.
\\[.2cm]
\noindent{\bf Main contributions.}\\[.1cm]
\noindent $\bullet$ We show that the minimum number of fundamental matrices
    that can be used to recover a configuration of $n$ cameras is always $\lceil (11n - 15)/7 \rceil$ (Theorem~\ref{thm:minimum_number}).\\
\noindent $\bullet$ We present several criteria for deciding whether or not a
    viewing graph is solvable. After revisiting some results
    from~\cite{levi2003viewing,rudi2010linear} (Section~\ref{sec:simple}), we
    describe a new necessary condition for solvability that is based on the
    number of edges and vertices of subgraphs (Theorem~\ref{thm:diff}), as well
    as a sufficient condition based on ``moves'' for adding new edges to a graph
    (Theorem~\ref{thm:moves}).\\
\noindent $\bullet$ We describe an algebraic formulation for solvability that
    in principle can always be used to determine whether any viewing graph is
    solvable. Although this method is computationally challenging for larger
    graphs, we also introduce a much more practical {\em linear} test, that can
    be used to verify whether a viewing graph identifies a finite number of
    camera configurations (Section~\ref{sec:linear}).\\
\noindent $\bullet$ Using an implementation of all the proposed methods, we
    analyze solvability for all minimal viewing graphs with at most $9$
    vertices. We also discuss some relevant examples
    (Section~\ref{sec:experiments}).


\section{Background}
To make our presentation mostly self-contained, we recall some basic
theoretical facts that are used in the rest of the paper.\\[.4cm]
\noindent {\bf Notation.} We write $\PP^n = \PP(\RR^{n+1})$ for the
$n$-dimensional real projective space. We use bold font for vectors and
matrices, and normal font for projective objects. For example, a point in
$\PP^3$ will be written as $p = [\vect p]$ where $\vect p$ is a vector in
$\RR^{4}$ and $p$ is the equivalence class associated with $\vect p$.
Similarly, a projective transformation represented by a matrix $\vect M$ will
be written as $M = [\vect M]$. We use $GL(n,\RR)$ for the group of $n \times
n$ invertible real matrices.

\subsection{Camera configurations and epipolar geometry}

A projective camera $P = [\vect P]$ is represented by $3\times 4$ matrix $\vect
P$ of full rank, defined up to scale. The matrix $\vect P$ describes a linear
projection $\PP^3\setminus \{c\} \rightarrow \PP^2$ where $c = [\vect c]$ is
the pinhole of the camera, associated with the null-space of $\vect P$.

The matrix group $GL(4,\RR)$ acts on the set of cameras by multiplication on
the right, and represents the group of projective transformations of $\PP^3$,
or of changes of homogeneous coordinates. We will use the fact that the group
of matrices in $GL(4,\RR)$ that fix a camera $P = [\vect P]$ with pinhole $c =
[\vect c]$ is given by
\begin{equation}\label{eq:stab_P} {\rm Stab}(P) = \{\alpha{\vect I}_4 + \vect
c \vect v^T \, | \, \alpha \in \RR\setminus \{0\}, \vect v \in \RR^4\} \cap
GL(4,\RR).
\end{equation}

Here ${\rm Stab}(P)$ stands for ``stabilizer''. Indeed, all the solutions for
$\vect M$ in $\vect P \vect M = \alpha \vect P$ are described
by~\eqref{eq:stab_P}. Note that ${\rm Stab}(P)$ only depends on the pinhole of
$P$. The following important fact follows directly from the form of ${\rm
Stab}(P)$.

\begin{lemma}\label{lemma:joint_stabilizer} Given two cameras $P_1$, $P_2$
with distinct pinholes, we have that
\begin{equation} {\rm Stab}(P_1) \cap {\rm Stab}(P_2) = \{\alpha {\vect I}_4,
\, | \, \alpha
\in \RR\setminus \{0\}\}.
\end{equation} In other words, the identity is the only projective
transformation that fixes both $P_1$ and $P_2$.
\end{lemma} 

Two sets of cameras $(P_1,\ldots,P_n)$ and $(P_1',\ldots,P_n')$
with $P_i = [\vect P_i]$, $P_i' = [\vect P_i']$ are {\em projectively
equivalent} if there exists a single projective transformation $T$ such that
$P_i = P_i' T$ (so if $T = [\vect T]$  with $\vect T$ in $GL(4,\RR)$, then
$\vect P_i =
\alpha_i \vect P_i' \vect T$ for non-zero constants $\alpha_i$). The set of
{\em configurations} of $n$ cameras is the set of $n$-tuples of cameras up to
projective equivalence. For any $n \ge 2$, the space of camera configurations
can be viewed as a manifold of dimension $11n - 15$.

Given two cameras $P_1 = [\vect P_1]$, $P_2 = [\vect P_2]$, the associated
{\em fundamental matrix} $F(P_1,P_2) = [\vect F]$ can be defined as the $3
\times 3$ matrix (up to scale) with entries
\begin{equation}\label{eq:fund_mat} f_{il} = (-1)^{i+l} \det (\vect P_{1j}^T
\vect P_{1k}^T \vect P_{2m}^T \vect P_{2n}^T),
\end{equation} where $\vect P_{ar}$ denotes the $r$-th row of $\vect P_a$, and
$(i,j,k)$ and $(l,m,n)$ are triples of distinct indices. The fundamental
matrix can be used to characterize pairs of corresponding points in the two
images, since $u_1 = [\vect u_1]$ and $u_2 = [\vect u_2]$ are projections of
the same 3D point if and only if $\vect u_1^T
\vect F \vect u_2 = 0$. For our purposes, 
the most important property of the fundamental matrix is that it is invariant
under projective transformations, and that $F(P_1,P_2)$ uniquely identifies
the configuration of $P_1$ and $P_2$~\cite[Theorem 9.10]{hartley2003multiple}.

Finally, viewed as a subset of $\PP^8$, the (closure of the) set of all
fundamental matrices forms a hypersurface defined by $\det(\vect F) = 0$. If
$F(P_1,P_2)= [\vect F]$, the left and right null-space of $\vect F$ represent
the two {\em epipoles} $e_{12} = P_1 c_2$ and $e_{21} = P_2 c_1$, which
are the images of each pinhole viewed from the other camera. An epipole
accounts for two of seven degrees of freedom of a fundamental matrix. In fact,
the information encoded in the fundamental matrix can be seen as the pair of
epipoles $e_{12}, e_{21}$, together with a projective transformation $\PP^1
\rightarrow \PP^1$ (known as ``epipolar line
homography''~\cite{hartley2003multiple}) between lines containing $e_{12}$ in
the first image and the lines containing $e_{21}$ in the second image. In
particular, the knowledge of two epipoles together with three point
correspondences completely determines a fundamental matrix.

\section{The viewing graph}

The viewing graph is a graph in which vertices correspond to cameras, and edges
represent fundamental matrices between them. More precisely, if $G = (V_G,E_G)$
is an undirected graph with $n$ vertices, and $P_1,\ldots, P_n$ are projective
cameras, we write
\begin{equation}
\mathcal F_G(P_1,\ldots,P_n) = \{F_{ij} = F(P_i,P_j) \, | \, (i,j) \in E_G\},
\end{equation} for the set of fundamental matrices defined by the edges of
$G$. We say that the the set $\mathcal F_G(P_1,\ldots,P_n)$ is {\em solvable}
if $\mathcal F_G(P_1,\ldots,P_n) = \mathcal F_G(P_1',\ldots,P_n')$ implies
that $(P_1,\ldots,P_n)$ and $(P_1',\ldots,P_n')$ are in the same projective
configuration. In other words, a set of fundamental matrices is solvable if
and only if it uniquely determines a projective configuration of cameras.

\begin{proposition}\label{prop:solvable_graph} The solvability of $\mathcal
F_G(P_1,\ldots,P_n)$ only depends on the graph $G$ and on the pinholes
$c_1,\ldots, c_n$ of $P_1,\ldots,P_n$. 
\end{proposition}

\begin{proof} The statement expresses the fact that changes of image
coordinates are only a “relabeling” of a camera configuration and the
associated fundamental matrices. More precisely, if $S_1,\ldots,S_n$ are
arbitrary projective transformations of $\PP^2$, then $(P_1, \ldots, P_n)$ and
$(P_1',\ldots,P_n')$ are in the same configuration if and only if the same is
true for $(S_1 P_1, \ldots, S_n P_n)$ and $(S_1 P_1',\ldots,S_n P_n')$. This
implies that $\mathcal F_G(P_1,\ldots,P_n)$ is solvable if and only $\mathcal
F_G(S_1 P_1,\ldots,S_n P_n)$ is.
\hfill \qed
\end{proof}

\begin{example}\label{ex:aligned} If $G$ is a {\em complete} graph with $n \ge 3$ vertices,
then $\mathcal F_G(P_1,\ldots,P_n)$ is solvable if and only if the pinholes of
the cameras $P_1,\ldots,P_n$ are not all aligned. Indeed, if the pinholes are
aligned, then the fundamental matrices between {\em all} pairs of cameras are
not sufficient to completely determine the configuration: replacing any $P_i =
[\vect P_i]$ with $P_i' = [\vect P_i (\vect I_4 + \vect c_j \vect v^T)]$, where
$c_j = [\vect c_j]$ is the pinhole of another camera and $\vect v^T$ is
arbitrary, yields a new set of cameras which belongs to a different
configuration but has the same set of fundamental matrices. Conversely, it is
known (see for example~\cite{heyden1997algebraic,trager2015joint}) that the
complete set of fundamental matrices determines a unique camera configuration
whenever there are at least three non-aligned pinholes. 
\hfill $\diamondsuit$
\end{example}

In the rest of the paper we will only consider generic configurations of
cameras/pinholes (so a complete graph will always be solvable). This covers
most cases of practical interest, although in the future degenerate
configurations (including some collinear or coplanar pinholes) could
be studied as well.

 
\begin{definition} A viewing graph $G $ is said to be \emph{solvable} if
$\mathcal F_G(P_1,\ldots,P_n)$ is solvable for generic cameras
$P_1,\ldots,P_n$.
\end{definition}

In other words, solvable viewing graphs describe sets of fundamental matrices
that are generically sufficient to recover a camera configuration. Despite its
clear significance, the problem of characterizing which viewing graphs are
solvable has not been studied much, and only partial answers are available in
the literature (mainly in~\cite{levi2003viewing,rudi2010linear}). It is quite
easy to produce examples of graphs that are solvable, but it is much more
challenging, given a graph, to determine whether it is solvable or not. The
following observation provides another useful formulation of solvability (note
that the ``if'' part requires the genericity assumption, as shown in
Example~\ref{ex:aligned}).

\begin{lemma} A viewing graph $G$ is solvable if and only if, for generic
cameras $P_1,\ldots,P_n$, the fundamental matrices $\mathcal
F_G(P_1,\ldots,P_n) = \{F(P_i,P_j) \, | \, (i,j) \in E_G \}$ uniquely
determine the remaining fundamental matrices $\{F(P_i,P_j) \, | \,
(i,j) \not \in E_G \}$.
\end{lemma}

This viewpoint also suggests the idea that, given any graph $G$, we can define
a ``solvable closure'' $\overline G$, as the graph obtained from $G$ by adding
edges corresponding to fundamental matrices that can be deduced generically
from $\mathcal F_G(P_1,\ldots,P_n)$. Hence, a graph is solvable if and only if
its closure is a complete graph. We will return to this point in
Section~\ref{sec:constructive}.

\subsection{Simple criteria}
\label{sec:simple}

We begin by recalling two necessary conditions for solvability that were shown
in~\cite{levi2003viewing}. These provide simple criteria to show that a viewing
graph is {\em not} solvable.

\begin{proposition}\cite{levi2003viewing}\label{prop:levi} If a viewing graph
with $n>3$ vertices is solvable, then: 1)~All vertices have degree at least 2. 2)~No two adjacent vertices have degree 2.
\end{proposition}

We extend this result with the following necessary condition (which implies
the first point in the previous statement).

\begin{proposition} Any solvable graph is 2-connected, \ie, it has the
property that after removing any vertex the graph remains connected.
\end{proposition}
\begin{proof} Assume that a vertex $i$ disconnects the graph $G$ into two
components $G_1, G_2$, and let $P_1,\ldots,P_n$ be a set of $n$ generic
cameras, whose pairwise fundamental matrices are represented by the edges of
$G$. If $c_i = [\vect c_i]$ is the pinhole of the camera $P_i$ associated with
$i$, then we consider two distinct projective transformations of the form $T_1
= [\vect I_4 + \alpha_1 \vect c_i \vect v_1^T]$ and $T_2 = [\vect I_4 +
\alpha_2
\vect c_i \vect v_2^T]$. These transformations fix the camera $P_i$. If we
apply $T_1$ to all cameras in $G_1$ and $T_2$ to all cameras in $G_2$, while
leaving $P_i$ fixed, we obtain a different camera configuration that gives
rise to the same set of fundamental matrices as $P_1,\ldots,P_n$ for all edges
in $G$. \hfill \qed
\end{proof}

We also recall a result from~\cite{rudi2010linear} which will be used in the
next section.

\begin{proposition}\cite{rudi2010linear} \label{prop:construction} If $G_1$ and $G_2$ are solvable
viewing graphs, then the graph $G$ obtained by identifying two vertices from
$G_1$ and with two from $G_2$ is solvable.
\end{proposition}

Note that if both pairs of vertices in the previous statement are connected by
edges in $G_1$ and $G_2$, then these two edges will automatically be identified
in~$G$.

\subsection{How many fundamental matrices?}

We now ask ourselves what is the minimal number of edges that a graph must have
to be solvable (or, equivalently, how many fundamental matrices are required to
recover a camera configuration). Since a single epipolar relation provides at
most $7$ constraints in the $(11n-15)$-dimensional space of camera configurations,
we deduce that any solvable graph must have {\em at least} $e(n) = \lceil (11 n
- 15)/7)
\rceil$ edges. This fact was previously observed in~\cite[Theorem
2]{rudi2010linear}. However, compared to~\cite{rudi2010linear}, we show here
that this bound is tight, \ie, that there always exists a solvable graph with
$e = e(n)$ edges. Concretely, this means that, for $n$ generic views, there is
always a way of recovering the corresponding camera configuration using $e(n)$
fundamental matrices.

\begin{theorem}\label{thm:minimum_number} The minimum
number of edges of a solvable viewing graph with $n\ge 2$ views is
\[ e(n) = \left\lceil \frac{11n-15}{7} \right \rceil.
\]
\end{theorem}
\begin{proof} For $n\le 9$, examples of solvable viewing graphs with $e(n)$
edges are illustrated Figure~\ref{fig:solvable_9}. The solvability of these
graphs will be shown in Section~\ref{sec:constructive} (all but one of these
also appear in~\cite{levi2003viewing}). In particular, let $G_0$ be a solvable
viewing graph with $9$ vertices and $12$ edges. Using
Proposition~\ref{prop:construction}, we deduce that, starting from a solvable
viewing graph $G$ with $n$ vertices and $e$ edges, we can always construct a
solvable graph $G'$ with $n+7$ vertices and $e
+ 11$ edges. The graph $G'$ is simply obtained by merging $G$ and $G_0$ as in
  Proposition~\ref{prop:construction}, using two pairs of vertices both
  connected by edges.

Now, for any $n>9$, we consider the unique integers $q, r$ such that $n=7q+ r$
and $2 \le r \le 8$. It is easy to see that
\[ e(n) = \left\lceil \frac{11n-15}{7} \right \rceil = 11q + \left\lceil
\frac{11r-15}{7} \right \rceil.
\] To obtain a solvable viewing graph with $n$ vertices and $e(n)$ edges, we
start from a solvable graph with $r$ vertices and $e(r)$ edges, and repeat the
gluing construction described above $q$ times. The resulting graph is solvable
and has the desired number of vertices and edges.
\hfill \qed
\end{proof}

\begin{remark} It is worth pointing out that, in order to recover projection
matrices for $n$ views, it is quite common to use $2n-3$ fundamental matrices
(see for example~\cite[Sec.4.4]{heyden2000tensorial}). In fact, as shown
in~\cite[Proposition 7]{trager2015joint}, a large class of solvable viewing
graphs can be defined, starting for example from a $3$-cycle, by adding
vertices of degree two, one at the time: this always gives a total of $2n-3$
edges. For this type of viewing graphs it is possible to recover projection
matrices incrementally, using a pair of fundamental matrices for each new camera.
In fact, it is probably quite often erroneously believed that $2n-3$ is the
minimal number of fundamental matrices that are required for multi-view
reconstruction. Part of the confusion may arise from the fact that the ``joint
image''~\cite{triggs1995matching,trager2015joint,aholt2013hilbert}, which
characterizes multi-view point correspondences in $(\PP^2)^n$, has dimension
three (or codimension $2n-3$). This means means that we expect $2n-3$
conditions to be necessary to cut out generically the set of image
correspondences among $n$ views. On the other hand, according to
Theorem~\ref{thm:minimum_number}, fewer constraints are actually sufficient to
determine camera geometry.\footnote{This implies however that fewer than
$2n-3$ conditions can in fact determine a joint image in $(\PP^2)^n$, at least ``indirectly'' through the camera configuration. Mathematically, this is an interesting phenomenon that could be investigated in the future.}%
\end{remark}
Some values of $e(n)$ are listed in Table~\ref{tab:e(n)} (here $d(n)$
represents the minimal number of constraints on the fundamental matrices, and
will be discussed in the next section). Note that $e(n) < 2n -3$ for all
$n\ge5$.
\begin{table}
\scriptsize
\centering
\caption{The relation between $n$, $e(n)$, and $d(n)$}
\begin{tabular}{ @{} c c c c c c c c c c c c c c c c @{}} 
$n$ & 2 & 3 & 4 & 5 & 6 & 7
& 8 & 9 & 10 & 11 & 12 & 13 & 14 & 15 & 16 \\ \midrule 
$e(n)$ & 1 & 3 & 5 & 6 & 8 & 9 & 11 & 12 & 14 & 16 & 17 & 19 & 20 & 22 & 23 \\ \midrule
$d(n)$ & 0 & 3 & 6 & 2 & 5 & 1 & 4 & 0 & 3 & 6 & 2 & 5 &
1 & 4 & 0 \\ 
\end{tabular}
\label{tab:e(n)}
\end{table}
\begin{figure}[htbp]
    \centering    
    \begin{tikzpicture}[x=1.1cm, y=1.1cm, step=1] 
\node[myNode](0) at (0,0){};
\node[myNode](1) at (0,1){};
\node[myNode](2) at (1,1){};
\node[myNode](3) at (1,0){};
\draw[myEdge] (0) -- (1) -- (2) -- (3) -- (0) -- (2);

\node[below] at (.5,0){\scriptsize $n=4$}; 
\end{tikzpicture}
    \begin{tikzpicture}[x=1.3cm, y=1.3cm, step=1] 
\node[myNode](0) at (0,-.1){};
\node[myNode](1) at (.6,.5){};
\node[myNode](2) at (0,1.1){};
\node[myNode](3) at (-.6,.5){};
\node[myNode](4) at (.2,.5){};
\draw[myEdge] (0) -- (1) -- (2) -- (3) -- (0) -- (4) -- (2); 
\node[below] at (0,-.15){\scriptsize $n=5$}; 
\end{tikzpicture}
    \begin{tikzpicture}[x=.6cm, y=.6cm, step=1] 
\node[myNode](0) at (0,-1){};
\node[myNode](1) at (-1,0){};
\node[myNode](2) at (0,1){};
\node[myNode](3) at (1,0){};
\node[myNode](4) at (-.9,1.5){};
\node[myNode](5) at (.76,1.5){};
\draw[myEdge] (0) -- (1) -- (2) -- (3) -- (0); 
\draw[myEdge] (1) -- (4) -- (2); 
\draw[myEdge] (4) -- (5) -- (3); 

\node[below] at (0){\scriptsize $n=6$}; 
\end{tikzpicture}
    \begin{tikzpicture}[x=.8cm, y=.8cm, step=1] 
\node[myNode](0) at (0,0){};
\node[myNode](1) at (0,1){};
\node[myNode](2) at (-1,.5){};
\node[myNode](3) at (-1,-.5){};
\node[myNode](4) at (0,-1){};
\node[myNode](5) at (1,-.5){};
\node[myNode](6) at (1,.5){};
\draw[myEdge] (0) -- (1) -- (2) -- (3) -- (0);
\draw[myEdge] (0) -- (5) -- (4) -- (3); 
\draw[myEdge] (1) -- (6) -- (5); 

\node[below] at (0,-1){\scriptsize $n=7$}; 
\end{tikzpicture}
    \begin{tikzpicture}[x=.8cm, y=.8cm, step=1] 
\node[myNode](0) at (0,0){};
\node[myNode](1) at (0,1){};
\node[myNode](2) at (-1,.5){};
\node[myNode](3) at (-1,-.5){};
\node[myNode](4) at (0,-1){};
\node[myNode](5) at (1,-.5){};
\node[myNode](6) at (1,.5){};
\node[myNode](7) at (.5,.25){};
\draw[myEdge] (0) -- (1) -- (2) -- (3) -- (0);
\draw[myEdge] (0) -- (5) -- (4) -- (3); 
\draw[myEdge] (1) -- (6) -- (5); 
\draw[myEdge] (1) -- (7) -- (5);

\node[below] at (0,-1){\scriptsize $n=8$}; 

\end{tikzpicture}
    \begin{tikzpicture}[x=.8cm, y=.8cm, step=1] 
\node[myNode](0) at (0,1){};
\node[myNode](1) at (-1,.5){};
\node[myNode](2) at (-1,-.5){};
\node[myNode](3) at (0,-1){};
\node[myNode](4) at (1,-.5){};
\node[myNode](5) at (1,.5){};
\node[myNode](6) at (.5,.25){};
\node[myNode](7) at (-.5,.25){};
\node[myNode](8) at (0,-.5){};
\draw[myEdge] (0) -- (1) -- (2) -- (3) -- (4) -- (5) -- (0);
\draw[myEdge] (0) -- (6) -- (4) -- (8) -- (2) -- (7) -- (0);

\node[below] at (0,-1){\scriptsize $n=9$}; 
\end{tikzpicture}
    \caption{Examples of minimal solvable viewing graphs for $n \le 9$ views (see Section~\ref{sec:constructive})}
    \label{fig:solvable_9}
\end{figure}
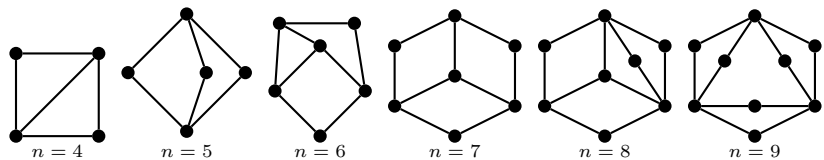


\subsection{Constraints on fundamental matrices}

Closely related to the solvability of viewing graphs is the problem of
describing {\em compatibility} of fundamental matrices. Indeed, given a
solvable graph $G$, it is not true in general that any set of fundamental
matrices can be assigned to the edges of $G$, since fundamental matrices must
satisfy some feasibility constraints in order to correspond to an actual
camera configuration. For example, it is well known that the fundamental
matrices $F_{12}, F_{23}, F_{31}$ relating three pairs of cameras with 
non-aligned pinholes are compatible if and only if
\begin{equation}\label{eq:compatible_triple}
\vect e_{13}^T \vect F_{12} \vect e_{23} = \vect e_{21}^T \vect F_{23} \vect
e_{31} = \vect e_{32}^T \vect F_{31} \vect e_{12} = 0,
\end{equation} where $e_{ij} = [\vect e_{ij}]$ is the epipole in image $i$
relative to the camera $j$~\cite[Theorem 15.6]{hartley2003multiple}. In most
practical situations fundamental matrices are estimated separately, so these
constraints need to be taken into account \cite{sweeney2015optimizing}.
However, it is sometimes incorrectly stated that compatibility for any set of
fundamental matrices only arises from triples and equations of the
form~\eqref{eq:compatible_triple}~\cite[Theorem 1]{rudi2010linear},
\cite[Definition 1]{sweeney2015optimizing}. While it is true that for a
complete set of $\binom{n}{2}$ fundamental matrices triple-wise compatibility
is sufficient to guarantee global compatibility, for smaller sets of
fundamental matrices other types of constraints will be necessary. For example,
there are many solvable viewing graphs with no three-cycles (\eg, the graph in
Figure~\ref{fig:solvable_9} with $n=5$), however the fundamental
matrices cannot be unconstrained if $7e(n) > 11n - 15$, which always true
unless $e = 2$ modulo $9$ (cf. Table~\ref{tab:e(n)}).

More formally, we can consider the set $\mathcal X$ of compatible fundamental
matrices between {\em all pairs} of $n$ views, so that $\mathcal X
\subset (\PP^8)^N$ where $N = \binom{n}{2}$. Since each compatible $N$-tuple
is associated with a camera configuration, we see that $\mathcal X$ has
dimension $11n - 15$. Given a viewing graph $G$ with $n$ views, we write
$\mathcal X_G \subset (\PP^8)^e$ for the projection of $\mathcal X$ onto the
factors in $(\PP^8)^N$ corresponding to the edges of $G$. The set $\mathcal
X_G$ thus represents compatible fundamental matrices for pairs of views
associated with the edges of $G$.
The following result follows from dimensionality arguments (see
the supplementary material for a complete proof).

\begin{proposition}\label{prop:dim}
If $G$ is solvable with $n$ vertices, $\mathcal X_G$ has dimension $11n-15$.
\end{proposition}

If $\mathcal X_G$ has dimension $11n-15$, then the fundamental matrices
assigned to the edges of $G$ must satisfy $d(n,e) = 7e-11n+15$
constraints\footnote{This is the codimension of $\mathcal X_G$ in $\mathcal
H^e$ where $\mathcal H \subset \PP^8$ is the determinant hypersurface.}. This
also means that the minimum number of constraints on the fundamental matrices
associated with a solvable graph is $d(n) = d(n,e(n))$ (see
Table~\ref{tab:e(n)}).

We now use Proposition~\ref{prop:dim} to deduce a new necessary condition for solvability.


\begin{theorem}\label{thm:diff} Let $G$ be a solvable graph with $n$
vertices and $e$ edges. Then for any subgraph $G'$ of $G$ with $n'$ vertices
and $e'$ edges we must have
\begin{equation}\label{eq:diff} d(n',e') \le d(n,e),
\end{equation} where $d(n,e) = 7e-11n+15$. More generally, if $G_1,\ldots,
G_k$ are subgraphs of $G$, each with $n_i$ vertices and $e_i$ edges, with the
property that the edge sets $E_{G_i} \subset E_G$ are pairwise disjoint, then
we must have
\begin{equation}\label{eq:diff2}
\sum_{i=1}^k d(n_i,e_i) \le d(n,e).
\end{equation}
\end{theorem}
\begin{proof} Using the same notation as above, we note that that $\mathcal
X_{G'}$ is a projection of $\mathcal X_{G}$ onto $e'$ factors of $(\PP^8)^e$:
this implies $\dim \mathcal X_{G'} + 7(e-e') \geq \dim \mathcal X_G$, or $7e' -
\dim \mathcal X_{G'} \le 7e - \dim \mathcal X_{G}$. Since $\dim
\mathcal X_{G'} \le 11 n' - 15$ and $\dim \mathcal X_{G} = 11n - 15$ (because $G$ is solvable), we obtain
\[ 7e' - 11n' + 15 \le 7e' - \dim \mathcal X_{G'} \le 7e - \dim \mathcal
X_{G} = 7e - 11n + 15.\]
For the second statement, we consider the graph $G' = (\bigcup_i V_{G_i},
\bigcup_i E_{G_i})$. Since the edges of $G_i$ are disjoint, we have
\[\dim \mathcal X_{G'} \le  \sum_{i=1}^k \dim \mathcal X_{G_i} \le
\sum_{i=1}^k (11 n_i - 15),\] 
and $e' = \sum_{i}^k e_i$. The result
follows again from $7e' - \dim \mathcal X_{G'} \le 7e - \dim \mathcal
X_{G}$.
\hfill \qed
\end{proof}

\begin{example}\label{ex:levi} In~\cite{levi2003viewing}, Levi and Werman
observe that all viewing graphs of the form shown in
Figure~\ref{fig:4_non_solvable} are not solvable. This can be easily deduced
from Theorem~\ref{thm:diff}. Indeed, for a graph $G$ of this form, the
subgraphs $G_1,G_2,G_3,G_4$ have disjoint edges, however we have (using the
same notation as in the proof of Theorem~\ref{thm:diff})
\[
\sum_{i=1}^4 d(n_i, e_i) = d(n, e) - 4\times 11 + 3 \times 15> d(n,e).
\] 
According to Theorem~\ref{thm:diff} this means that $G$ is not solvable.
\hfill $\diamondsuit$
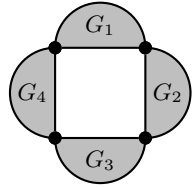
\begin{figure}[htbp]
    \centering
    \begin{tikzpicture}[x=1.2cm, y=1.2cm, step=1] 
\node[myNode](0) at (0,0){};
\node[myNode](1) at (0,1){};
\node[myNode](2) at (1,1){};
\node[myNode](3) at (1,0){};


\draw [myEdge, fill = lightgray] (0) arc [radius=.5, start angle=270, end angle=90] -- (0);
\draw [myEdge, fill = lightgray] (1) arc [radius=.5, start angle=180, end angle=0] -- (1);
\draw [myEdge, fill = lightgray] (2) arc [radius=.5, start angle=450, end angle=270] -- (2);
\draw [myEdge, fill = lightgray] (3) arc [radius=.5, start angle=360, end angle=180] -- (3);

\node[myNode](0) at (0,0){};
\node[myNode](1) at (0,1){};
\node[myNode](2) at (1,1){};
\node[myNode](3) at (1,0){};

\node at (.5,1.25) {$G_1$};
\node at (1.25,.5) {$G_2$};
\node at (.5,-.25) {$G_3$};
\node at (-.25,.5) {$G_4$};
\end{tikzpicture}
    \caption{A viewing graph of this form (where $G_1, G_2, G_3, G_4$ represent
    arbitrary subgraphs) is not solvable}
    \label{fig:4_non_solvable}
\end{figure}
\end{example}


\subsection{Constructive approach for verifying solvability}
\label{sec:constructive}

Until now we have mainly discussed necessary conditions for solvability, which
can be used to show that a given graph is not solvable. We next introduce a
general strategy for proving that a graph is solvable. This method is not
always guaranteed to work, but in practice it gives sufficient conditions
for most of the graphs we tested (cf. Section~\ref{sec:experiments}).

Recall from the beginning of this section that we introduced the ``viewing
closure'' $\overline G$ of $G$ as the graph obtained by adding to $G$ all
edges corresponding to fundamental matrices that can be deduced from $\mathcal
F_G(P_1,\ldots,P_n)$. Our approach consists of a series of ``moves'' which
describe valid ways to add new edges to a viewing graph. For this it is
convenient to introduce a new type of edge in the graph, which keeps track
of the fact that partial information about a fundamental matrix is available.
More precisely:

\begin{itemize}
    \item A {\em solid (undirected) edge} between vertices $i$ and $j$ means that
    the fundamental matrix between the views $i$ and $j$ is fixed (as before).
    \item A {\em directed dashed edge} (for short, a {\em dashed arrow})
    between vertices $i$ and $j$ means that the $i$-th epipole in the image $j$
    is fixed.
\end{itemize}

As these definitions suggest, a solid edge also counts as a dashed double-
arrow, but the converse is not true. We next introduce three basic ``moves''
(cf.~Figure~\ref{fig:moves}).

\begin{enumerate}[(I)]
    \item If there are solid edges defining a four-cycle with one diagonal,
    draw the other diagonal.
    \item If there are dashed arrows $1\rightarrow 2$, $1\rightarrow
    3$, and solid edges $2 - 4$ and $3 - 4$, draw a dashed arrow $1
    \rightarrow 4$.
    \item If there are double dashed arrows $1 \leftrightarrow 2$, together
    with three pairs of dashed arrows $i \rightarrow 1, i \rightarrow
    2$ for $i = 3,4,5$, make the arrow between $1$ and $2$ a solid (undirected) edge.
\end{enumerate}
    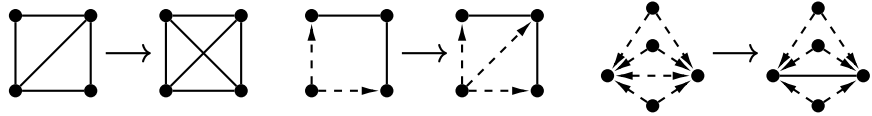
\begin{figure}[htbp]
        \centering
        \begin{tikzpicture}[x=1cm, y=1cm, step=1, baseline=.5cm] 
\node [myNode](0) at (0,0){};
\node [myNode](1) at (0,1){};
\node [myNode](2) at (1,1){};
\node [myNode](3) at (1,0){};

\draw[myEdge] (0) -- (1) -- (2) -- (3) -- (0) -- (2); 

\draw[thick,->] (1.2,.5) -- (1.8,.5);

\node [shift={(2 cm, 0 cm)}, myNode](0) at (0,0){};
\node [shift={(2 cm, 0 cm)}, myNode](1) at (0,1){};
\node [shift={(2 cm, 0 cm)}, myNode](2) at (1,1){};
\node [shift={(2 cm, 0 cm)}, myNode](3) at (1,0){};

\draw[myEdge] (0) -- (1) -- (2) -- (3) -- (0);
\draw[myEdge] (0) -- (2); 
\draw[myEdge] (1) -- (3);

\end{tikzpicture}\qquad
        \begin{tikzpicture}[x=1cm, y=1cm, step=1, baseline=.5cm] 

\node [myNode](0) at (0,0){};
\node [myNode](1) at (0,1){};
\node [myNode](2) at (1,1){};
\node [myNode](3) at (1,0){};

\draw[myEdge, dashed, -myTip] (0) -- (1); 
\draw[myEdge, dashed, -myTip] (0) -- (3);
\draw[myEdge] (1) -- (2);
\draw[myEdge] (3) -- (2);


\draw[thick,->] (1.2,.5) -- (1.8,.5);


\node [shift={(2 cm, 0 cm)}, myNode](0) at (0,0){};
\node [shift={(2 cm, 0 cm)}, myNode](1) at (0,1){};
\node [shift={(2 cm, 0 cm)}, myNode](2) at (1,1){};
\node [shift={(2 cm, 0 cm)}, myNode](3) at (1,0){};

\draw[myEdge, dashed, -myTip] (0) -- (1); 
\draw[myEdge, dashed, -myTip] (0) -- (3);
\draw[myEdge, dashed, -myTip] (0) -- (2);
\draw[myEdge] (1) -- (2);
\draw[myEdge] (3) -- (2);
\end{tikzpicture}\qquad
        \begin{tikzpicture}[x=1cm, y=1cm, step=1, baseline=.5cm] 

\node [myNode](0) at (-.6,.2){};
\node [myNode](1) at (.6,.2){};
\node [myNode](2) at (0,.6){};
\node [myNode](3) at (0,1.1){};
\node [myNode](4) at (0,-.2){};

\draw[myEdge, dashed, myTip-myTip] (0) -- (1); 
\draw[myEdge, dashed, -myTip] (2) -- (1);
\draw[myEdge, dashed, -myTip] (2) -- (0);
\draw[myEdge, dashed, -myTip] (3) -- (1);
\draw[myEdge, dashed, -myTip] (3) -- (0);
\draw[myEdge, dashed, -myTip] (4) -- (1);
\draw[myEdge, dashed, -myTip] (4) -- (0);


\draw[thick, ->] (.8,.5) -- (1.4,.5);


\node [shift={(2.2 cm, 0 cm)}, myNode](0) at (-.6,.2){};
\node [shift={(2.2 cm, 0 cm)}, myNode](1) at (.6,.2){};
\node [shift={(2.2 cm, 0 cm)}, myNode](2) at (0,.6){};
\node [shift={(2.2 cm, 0 cm)}, myNode](3) at (0,1.1){};
\node [shift={(2.2 cm, 0 cm)}, myNode](4) at (0,-.2){};

\draw[myEdge, -] (0) -- (1); 
\draw[myEdge, dashed, -myTip] (2) -- (1);
\draw[myEdge, dashed, -myTip] (2) -- (0);
\draw[myEdge, dashed, -myTip] (3) -- (1);
\draw[myEdge, dashed, -myTip] (3) -- (0);
\draw[myEdge, dashed, -myTip] (4) -- (1);
\draw[myEdge, dashed, -myTip] (4) -- (0);

\end{tikzpicture}
        \caption{Three moves (left: I, center: II, right: III) that can be
        used to prove solvability}
        \label{fig:moves}
    \end{figure}

\begin{theorem}\label{thm:moves} Let $G$ be a viewing graph. If applying the
three moves described above iteratively to $G$ we obtain a complete graph,
then $G$ is solvable.
\end{theorem}
\begin{proof} 
For each of the three moves we need to show that the new
edges contain information about the unknown fundamental matrices that can
actually be deduced from $\mathcal F_{G}(P_1,\ldots,P_n)$.
\begin{description}
    \item \underline{Move I}: The second diagonal of the square is deducible
    from the other edges because the square with one diagonal is a solvable
    graph (this is a simple consequence of Proposition~\ref{prop:construction}).
    \item \underline{Move II}: Assume that $e_{21} = P_2c_1$ and $e_{32} = P_3
    c_2$ are fixed epipoles in images $2$ and $3$, and that the fundamental
    matrices $F_{24} = F(P_2,P_4), F_{34} = F(P_3,P_4)$ are also fixed. If
    $c_1, c_2, c_4$ are not aligned, we can use $F_{24}$ to ``transfer'' the
    point $e_{21}$, and obtain a line $l_{41}$ in image $4$ that contains
    epipole $e_{41}$.
    Similarly, if $c_1, c_3, c_4$ are not aligned, we obtain another line
    $m_{41}$ using the same procedure with $F_{34}$ and $e_{31}$. If the
    pinholes $c_1,c_2,c_3,c_4$ are not all coplanar, the lines $l_{41}$ and
    $m_{41}$ will be distinct, and their point of intersection will be
    $e_{41}$. This implies that we can draw a dashed arrow from $1$ to $4$.
    \item \underline{Move III}: Assume that the epipoles $e_{21}$ and $e_{12}$
    are fixed, and that the images of three other pinholes $c_3,c_4,c_5$ are
    fixed in both images $1$ and $2$. If the planes $c_1,c_2,c_i$ for
    $i=3,4,5$ are distinct, then the images of $c_3,c_4,c_5$ give three
    correspondences that fix the epipolar line homography. This completely
    determines $F_{12}$, and we can draw a solid edge between $1$ and $2$.
    \hfill \qed
\end{description}
\end{proof}

In practice, the three moves can be applied cyclically until no new edges can
be added (it is also easy to argue the order is irrelevant, because we are
simply annotating information that is always deducible from the graph). 
Finally, we note that all three moves are {\em constructive} and {\em
linear}, meaning they actually provide a practical strategy for computing
all fundamental matrices: it is enough to transfer epipoles appropriately,
and use them to impose linear conditions on the unknown fundamental matrices.

\begin{example} Using Theorem~\ref{thm:moves}, we can show that all graphs from
Figure~\ref{fig:solvable_9} are solvable. Figure~\ref{fig:solvable_moves}
illustrates this explicitly for two cases ($n=6$ and $n=8$).
\hfill $\diamondsuit$
\end{example}

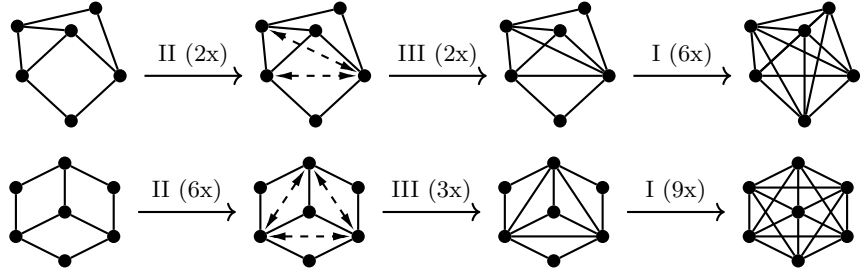
\begin{figure}[htbp]
    \centering
    \begin{tikzpicture}[x=.65cm, y=.6cm, step=1] 
\node[myNode](0) at (0,-1){};
\node[myNode](1) at (-1,0){};
\node[myNode](2) at (0,1){};
\node[myNode](3) at (1,0){};
\node[myNode](4) at (-1.1,1.1){};
\node[myNode](5) at (.5,1.5){};
\draw[myEdge] (0) -- (1) -- (2) -- (3) -- (0); 
\draw[myEdge] (1) -- (4) -- (2); 
\draw[myEdge] (4) -- (5) -- (3);

\draw[thick, ->] (1.5,0) -- (2.5,0) node[above]{II (2x)} -- (3.5,0);

\node[shift={(5, 0)}, myNode](0) at (0,-1){};
\node[shift={(5, 0)}, myNode](1) at (-1,0){};
\node[shift={(5, 0)}, myNode](2) at (0,1){};
\node[shift={(5, 0)}, myNode](3) at (1,0){};
\node[shift={(5, 0)}, myNode](4) at (-1.1,1.1){};
\node[shift={(5, 0)}, myNode](5) at (.5,1.5){};

\draw[myEdge] (0) -- (1) -- (2) -- (3) -- (0); 
\draw[myEdge] (1) -- (4) -- (2); 
\draw[myEdge] (4) -- (5) -- (3); 

\draw[myEdge, dashed, myTip-myTip] (1) -- (3);
\draw[myEdge, dashed, myTip-myTip] (3) -- (4);

\draw[shift={(5, 0)}, thick, ->] (1.5,0) -- (2.5,0) node[above]{III (2x)} -- (3.5,0);

\node[shift={(10, 0)}, myNode](0) at (0,-1){};
\node[shift={(10, 0)}, myNode](1) at (-1,0){};
\node[shift={(10, 0)}, myNode](2) at (0,1){};
\node[shift={(10, 0)}, myNode](3) at (1,0){};
\node[shift={(10, 0)}, myNode](4) at (-1.1,1.1){};
\node[shift={(10, 0)}, myNode](5) at (.5,1.5){};

\draw[myEdge] (0) -- (1) -- (2) -- (3) -- (0); 
\draw[myEdge] (1) -- (4) -- (2); 
\draw[myEdge] (4) -- (5) -- (3); 

\draw[myEdge] (1) -- (3);
\draw[myEdge] (3) -- (4);

\draw[shift={(10, 0)}, thick, ->] (1.5,0) -- (2.5,0) node[above]{I (6x)} -- (3.5,0);

\node[shift={(15, 0)}, myNode](0) at (0,-1){};
\node[shift={(15, 0)}, myNode](1) at (-1,0){};
\node[shift={(15, 0)}, myNode](2) at (0,1){};
\node[shift={(15, 0)}, myNode](3) at (1,0){};
\node[shift={(15, 0)}, myNode](4) at (-1.1,1.1){};
\node[shift={(15, 0)}, myNode](5) at (.5,1.5){};

\draw[myEdge] (0) -- (1);
\draw[myEdge] (2) -- (3);
\draw[myEdge] (3) -- (0); 
\draw[myEdge] (1) -- (4); 
\draw[myEdge] (4) -- (2); 
\draw[myEdge] (4) -- (5);
\draw[myEdge] (5) -- (3); 
\draw[myEdge] (1) -- (3);
\draw[myEdge] (3) -- (4);

\draw[myEdge] (0) -- (2);
\draw[myEdge] (0) -- (4);
\draw[myEdge] (0) -- (5);

\draw[myEdge] (1) -- (2);
\draw[myEdge] (1) -- (5);

\draw[myEdge] (2) -- (5);

\end{tikzpicture}\\[.35cm]
    \begin{tikzpicture}[x=.65cm, y=.65cm, step=1] 
\node[myNode](0) at (0,0){};
\node[myNode](1) at (0,1){};
\node[myNode](2) at (-1,.5){};
\node[myNode](3) at (-1,-.5){};
\node[myNode](4) at (0,-1){};
\node[myNode](5) at (1,-.5){};
\node[myNode](6) at (1,.5){};
\draw[myEdge] (0) -- (1) -- (2) -- (3) -- (0);
\draw[myEdge] (0) -- (5) -- (4) -- (3); 
\draw[myEdge] (1) -- (6) -- (5);

\draw[thick, ->] (1.5,0) -- (2.5,0) node[above]{II (6x)} -- (3.5,0);

\node[shift={(5, 0)}, myNode](0) at (0,0){};
\node[shift={(5, 0)}, myNode](1) at (0,1){};
\node[shift={(5, 0)}, myNode](2) at (-1,.5){};
\node[shift={(5, 0)}, myNode](3) at (-1,-.5){};
\node[shift={(5, 0)}, myNode](4) at (0,-1){};
\node[shift={(5, 0)}, myNode](5) at (1,-.5){};
\node[shift={(5, 0)}, myNode](6) at (1,.5){};

\draw[myEdge] (0) -- (1) -- (2) -- (3) -- (0);
\draw[myEdge] (0) -- (5) -- (4) -- (3); 
\draw[myEdge] (1) -- (6) -- (5);
\draw[myEdge, dashed, myTip-myTip] (1) -- (3);
\draw[myEdge, dashed, myTip-myTip] (3) -- (5);
\draw[myEdge, dashed, myTip-myTip] (5) -- (1);

\draw[shift={(5, 0)}, thick, ->] (1.5,0) -- (2.5,0) node[above]{III (3x)} -- (3.5,0);

\node[shift={(10, 0)}, myNode](0) at (0,0){};
\node[shift={(10, 0)}, myNode](1) at (0,1){};
\node[shift={(10, 0)}, myNode](2) at (-1,.5){};
\node[shift={(10, 0)}, myNode](3) at (-1,-.5){};
\node[shift={(10, 0)}, myNode](4) at (0,-1){};
\node[shift={(10, 0)}, myNode](5) at (1,-.5){};
\node[shift={(10, 0)}, myNode](6) at (1,.5){};

\draw[myEdge] (0) -- (1) -- (2) -- (3) -- (0);
\draw[myEdge] (0) -- (5) -- (4) -- (3); 
\draw[myEdge] (1) -- (6) -- (5);
\draw[myEdge] (1) -- (3);
\draw[myEdge] (3) -- (5);
\draw[myEdge] (5) -- (1);

\draw[shift={(10, 0)}, thick, ->] (1.5,0) -- (2.5,0) node[above]{I (9x)} -- (3.5,0);

\node[shift={(15, 0)}, myNode](0) at (0,0){};
\node[shift={(15, 0)}, myNode](1) at (0,1){};
\node[shift={(15, 0)}, myNode](2) at (-1,.5){};
\node[shift={(15, 0)}, myNode](3) at (-1,-.5){};
\node[shift={(15, 0)}, myNode](4) at (0,-1){};
\node[shift={(15, 0)}, myNode](5) at (1,-.5){};
\node[shift={(15, 0)}, myNode](6) at (1,.5){};

\draw[myEdge] (0) -- (1) -- (2) -- (3) -- (0);
\draw[myEdge] (0) -- (5) -- (4) -- (3); 
\draw[myEdge] (1) -- (6) -- (5);
\draw[myEdge] (1) -- (3);
\draw[myEdge] (3) -- (5);
\draw[myEdge] (5) -- (1);
\draw[myEdge] (0) -- (2);
\draw[myEdge] (0) -- (4);
\draw[myEdge] (0) -- (6);
\draw[myEdge] (2) -- (4) -- (6) -- (2);

\end{tikzpicture}
    \caption{Two applications of Theorem~\ref{thm:moves} to prove that viewing graphs are solvable}
    \label{fig:solvable_moves}
\end{figure}

\section{Algebraic tests for solvability and finite solvability}
\label{sec:linear}

Given a viewing graph $G$, it is possible to write down a set of algebraic
conditions that will in principle always determine whether $G$ is solvable. One
way to do this is by characterizing the set of projective transformations of
$\PP^3$ that can be applied to all cameras without affecting any of the
fundamental matrices represented by the edges of the viewing graph. More
precisely, since every pair of vertices connected by an edge represents a
projectively rigid pair of cameras, we assign a matrix $\vect g_{\lambda}$ in
$GL(4,\RR)$ to each edge $\lambda$ of the graph (so $\vect g_{\lambda}$
describes a projective transformation applied to a pair of cameras). We then
impose that matrices on adjacent edges act compatibly on the shared
vertex/camera. If the edges $\lambda$ and $\lambda'$ share a vertex $i$, then
from~\eqref{eq:stab_P} we see that this compatibility can be written as
\begin{equation}\label{eq:inv_stab}
\vect g_{\lambda} \vect g_{\lambda'}^{-1} = \alpha \vect I_4 + \vect c_i \vect
v^T,
\end{equation} where $\alpha$ is an arbitrary (nonzero) constant and $\vect v$
is an arbitrary vector. Thus, if $G$ is a viewing graph with $e$ edges
and $c_1,\ldots, c_n$ are a set of pinholes, we consider the set of all
compatible assignments of matrices:
\begin{small}
\[
\mathcal T_G({c_1,\ldots,c_n}) = \{(\vect g_{\lambda}, \lambda \in E_G) \, |
\, \eqref{eq:inv_stab} \mbox{ holds for all adjacent edges in G}\} \subset
GL(4,\RR)^e.
\] 
\end{small}
If $G$ is solvable, then for general $c_1,\ldots,c_n$ the set
$\mathcal T_G({c_1,\ldots,c_n})$ will consist of $e$-tuples of matrices that
are all scalar multiples of each other. This in fact means that the only way to
act on all cameras without affecting the fixed fundamental matrices is to apply
a single projective transformation.

By substituting random pinholes in~\eqref{eq:inv_stab}, we can use these
equations for $\mathcal T_G({c_1,\ldots,c_n})$ as an algebraic test for
verifying whether a viewing graph is solvable. This approach however is
computationally very challenging, since it requires solving a non-linear
algebraic system with a large number of variables. On the other hand, if we are
only interested in the {\em dimension} of $\mathcal T_G({c_1,\ldots,c_n})$,
then we can use a much simpler strategy: noting that $\mathcal
T_G({c_1,\ldots,c_n})$ may be viewed as an algebraic group (it is a subgroup of
$GL(4,\RR)^e$), it is sufficient to compute the dimension of its \emph{tangent
space} at any point, and in particular at the identity (\ie, the product of
identity matrices).\footnote{Here we actually need that $\mathcal
T_G({c_1,\ldots,c_n})$ is smooth: this follows from a technical result, which
states that an algebraic group (more properly a ``group scheme'') over a field
of characteristic zero is always smooth~\cite[Sec.11]{mumford2008abelian}.} An
explicit representation of the tangent space of $\mathcal
T_G({c_1,\ldots,c_n})$ is provided by the following result (see the
supplementary material for a proof).

\begin{proposition} The tangent space of $\mathcal T_G({c_1,\ldots,c_n})$ at
the identity can be represented as the space of $e$-tuple of matrices $(\vect
h_{\lambda}, \, \lambda
\in E_G)$ where each $\vect h_{\lambda}$ is in $\RR^{4 \times 4}$ (not
necessary invertible), and with compatibility conditions of the form
\begin{equation}\label{eq:comp_linear}
\vect h_{\lambda} - \vect h_{\lambda'} = \alpha \vect I_4 + \vect c_i \vect
v^T,
\end{equation} where $\alpha \in \RR \setminus \{0\}$ and $\vect v \in \RR^4$
are arbitrary, and $\lambda$ and $\lambda'$ share the vertex $i$.
\end{proposition}

When the pinholes have been fixed, the compatibility
constraints~\eqref{eq:comp_linear} can be expressed as {\em linear} equations
in the entries of the matrices $\vect h_{\lambda}$. These equations are
obtained by eliminating the variables $\alpha$ and $\vect v$
from~\eqref{eq:comp_linear}. The resulting conditions in terms of $\vect
h_{\lambda}, \vect h_{\lambda'}, \vect c_i$ are rather simple, and listed
explicitly in the supplementary material. Using this approach, the dimension of
$\mathcal T_G({c_1,\ldots,c_n})$ is easy to determine: it is enough to fix the
pinholes randomly, and compute the dimension of the induced linear system.

When $\mathcal T_G({c_1,\ldots,c_n})$ has dimension $d = 15 + e$ (which
accounts for the group of projective transformations, and scale factors for
each matrix $\vect g_{\lambda}$), we deduce that there are at most a {\em
finite} number of projectively inequivalent ways in which we can act on all the
cameras without affecting the fixed fundamental matrices. In other words, the
fundamental matrices associated with the edges of $G$ determine at most a
finite set of camera configurations (rather than a single configuration, which
is our definition for solvability). When this happens, we say that $G$ is
{\em finite solvable}. On the other hand, we were not able to find an example
of a finite solvable graph that is provably not solvable, nor to find
a proof that ``finite solvability'' implies ``solvability''. To our knowledge,
whether a set of fundamental matrices can characterize a finite number of
configurations, but more than a single one, is a question that has never been
addressed.


\begin{open} Is it possible for a viewing graph to be
finite solvable without being solvable?
\end{open}


Our experiments show that this behavior does not occur for a small number of
vertices, but we see no reason why this should be true for larger graphs. This
is certainly an important issue that we hope to investigate in the future.

\section{Experiments and examples}
\label{sec:experiments}

We have implemented and tested all of the discussed criteria and methods using
the free mathematical software {\tt SageMath}~\cite{sagemath}.
\footnote{Our code is available at
\url{https://github.com/mtrager/viewing-graphs}.} We then analyzed solvability for
all minimal viewing graphs with $n
\le 9$ vertices and $e(n) = \lceil (11n-15)/7 \rceil$ edges. The results are
summarized in Table~\ref{tab:experiments}. For every pair $(n,e(n))$, we list
the number of all non-isomorphic connected graphs of that size (``connected''),
the number of graphs that satisfy the necessary condition from
Theorem~\ref{thm:diff} (``candidates''), the number of those that satisfy the
sufficient condition from Theorem~\ref{thm:moves} (``solvable with moves''),
and the number of graphs that are finite solvable (``finite solvable''), using
the linear method from~\ref{sec:linear}. We see that Theorems~\ref{thm:diff}
and~\ref{thm:moves} allow us to recover all minimal solvable graphs for $n \le
7$, since candidate graphs are always solvable with moves. On the other hand,
for $n=8$, and particularly for the unconstrained case $n=9$, there are some
graphs that we could not classify with those methods (although finite 
solvability was easy to verify in all cases). For the undecided graphs, we were
sometimes able to prove solvability with the general algebraic method from
Section~\ref{sec:linear}, or using other arguments.  The following examples
present a few interesting cases.
\begin{table}
\centering
\label{tab:experiments}
\caption{Solvability of minimal viewing graphs using our methods}
\scriptsize
\begin{tabular}{@{} l c c c c c c c @{}}
\toprule
$(n,e(n))$ & (3,3) & (4,5) & (5,6) & (6,8) & (7,9) & (8,11) & (9,12) \\ \midrule
connected & 1 & 1 & 5 & 22 & 107 & 814 & 4495\\
candidates & 1 & 1 & 1 & 4 & 3 & 36 & 28\\
solvable with moves & 1 & 1 & 1 & 4 & 3 & 31 & 5\\
finite solvable & 1 & 1 & 1 & 4 & 3 & 36 & 27\\
\bottomrule
\end{tabular}
\end{table}
\begin{example}\label{ex:8dotted} The graph shown in
Figure~\ref{fig:examples_interesting} (left) is one of the five cases with
$n=8$, $e=11$ that are ``candidates'' but are not ``solvable with moves''.
However, we can show that this graph is actually solvable by arguing that the
image of the pinhole $1$ in the view $7$ is fixed, even if this is not a
consequence of the moves of Theorem~\ref{thm:moves} (this is represented by the
gray dashed arrow in the figure). To prove this fact, one needs to keep track
of more information, and record also when an epipole is constrained to a line
(rather than only when an epipole is fixed, which is the purpose of dashed
edges).\footnote{This information can be taken into account by defining a
new type of edge together with additional moves. We did not do this in
Theorem~\ref{thm:moves}
because this type of edge is never necessary for smaller
graphs.} After drawing the dashed arrow from $1$ to $7$, solvability
can be shown using the moves from Theorem~\ref{thm:moves}.
\hfill $\diamondsuit$
\end{example}

\begin{example}\label{ex:9positive_dim} The graph shown in
Figure~\ref{fig:examples_interesting} (center) is the only viewing graph with
$n=9$ and $e=12$ that is ``candidate'' but is not ``finite solvable''. The fact
that it is not finite solvable can also deduced without computations. Indeed,
any finite solvable viewing graph of this size cannot impose any constraints on
the fundamental matrices associated with its edges (this is because
$d(9,12)=7\times 12 - 11\times 9 + 15 = 0$). However, the image of the pinhole
$7$ in the view $2$ is over-constrained, because we can draw a dashed arrow
$7\rightarrow 2$ using move II for two distinct four-cycles (($7,1,2,5$) and
($7,3,2,5$)). This implies that the fundamental matrices associated with the
edges of the graph cannot be arbitrary.~\hfill $\diamondsuit$
\end{example}

\begin{example}\label{ex:solvable_argument} The graph shown in
Figure~\ref{fig:examples_interesting} (right) is not ``solvable with moves'',
however one can show that it is solvable: indeed, the general algebraic
compatibility equations from Section~\ref{sec:linear} are in this case simple
and can be solved explicitly (see the supplementary material for
the computations). The fundamental matrices associated with the edges of the
graph are unconstrained, 
so $12$ arbitrary fundamental matrices
determine a unique configuration of $9$~cameras.~\hfill
$\diamondsuit$
\end{example}
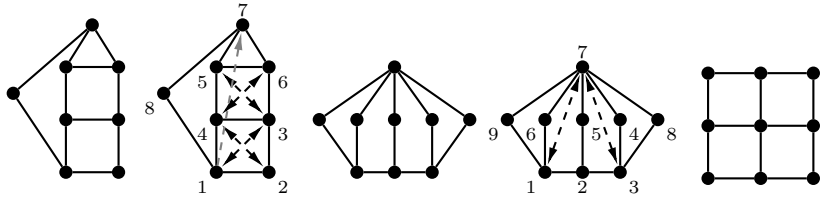
\begin{figure}[htbp]
    \centering
    \begin{tikzpicture}[x=.7cm, y=.7cm, step=1] 
\node[myNode](1) at (0,0){};
\node[myNode](2) at (1,0){};
\node[myNode](3) at (1,1){};
\node[myNode](4) at (0,1){};
\node[myNode](5) at (0,2){};
\node[myNode](6) at (1,2){};
\node[myNode](7) at (0.5,2.8){};
\node[myNode](8) at (-1,1.5){};
\draw[myEdge] (1) -- (2) -- (3) -- (4) -- (1);
\draw[myEdge] (4) -- (5) -- (6) -- (3); 
\draw[myEdge] (5) -- (7) -- (6); 
\draw[myEdge] (7) -- (8) -- (1);


\node[shift={(2 cm, 0 cm)}, myNode](1) at (0,0){};
\node[shift={(2 cm, 0 cm)}, myNode](2) at (1,0){};
\node[shift={(2 cm, 0 cm)}, myNode](3) at (1,1){};
\node[shift={(2 cm, 0 cm)}, myNode](4) at (0,1){};
\node[shift={(2 cm, 0 cm)}, myNode](5) at (0,2){};
\node[shift={(2 cm, 0 cm)}, myNode](6) at (1,2){};
\node[shift={(2 cm, 0 cm)}, myNode](7) at (0.5,2.8){};
\node[shift={(2 cm, 0 cm)}, myNode](8) at (-1,1.5){};
\draw[myEdge] (1) -- (2) -- (3) -- (4) -- (1);
\draw[myEdge] (4) -- (5) -- (6) -- (3); 
\draw[myEdge] (5) -- (7) -- (6); 
\draw[myEdge] (7) -- (8) -- (1);
\draw[myEdge,dashed, myTip-myTip] (1) -- (3);
\draw[myEdge,dashed, myTip-myTip] (2) -- (4);
\draw[myEdge,dashed, myTip-myTip] (3) -- (5);
\draw[myEdge,dashed, myTip-myTip] (4) -- (6);
\draw[myEdge,dashed,-myTip,gray] (1) -- (7);

\node[below left] at (1){\scriptsize 1};
\node[below right] at (2){\scriptsize 2};
\node[below right] at (3){\scriptsize 3};
\node[below left] at (4){\scriptsize 4};
\node[below left] at (5){\scriptsize 5};
\node[below right] at (6){\scriptsize 6};
\node[above] at (7){\scriptsize 7};
\node[below left] at (8){\scriptsize 8};

\end{tikzpicture}
    \begin{tikzpicture}[x=.5cm, y=.7 cm, step=1] 
\node[myNode](1) at (-1,0){};
\node[myNode](2) at (0,0){};
\node[myNode](3) at (1,0){};
\node[myNode](4) at (1,1){};
\node[myNode](5) at (0,1){};
\node[myNode](6) at (-1,1){};
\node[myNode](7) at (0,2){};
\node[myNode](8) at (2,1){};
\node[myNode](9) at (-2,1){};
\draw[myEdge] (1) -- (2) -- (3);
\draw[myEdge] (1) -- (6) -- (7); 
\draw[myEdge] (2) -- (5) -- (7); 
\draw[myEdge] (3) -- (4) -- (7); 
\draw[myEdge] (1) -- (9) -- (7); 
\draw[myEdge] (3) -- (8) -- (7);

\node[shift={(5, 0)}, myNode](1) at (-1,0){};
\node[shift={(5, 0)}, myNode](2) at (0,0){};
\node[shift={(5, 0)}, myNode](3) at (1,0){};
\node[shift={(5, 0)}, myNode](4) at (1,1){};
\node[shift={(5, 0)}, myNode](5) at (0,1){};
\node[shift={(5, 0)}, myNode](6) at (-1,1){};
\node[shift={(5, 0)}, myNode](7) at (0,2){};
\node[shift={(5, 0)}, myNode](8) at (2,1){};
\node[shift={(5, 0)}, myNode](9) at (-2,1){};
\draw[myEdge] (1) -- (2) -- (3);
\draw[myEdge] (1) -- (6) -- (7); 
\draw[myEdge] (2) -- (5) -- (7); 
\draw[myEdge] (3) -- (4) -- (7); 
\draw[myEdge] (1) -- (9) -- (7); 
\draw[myEdge] (3) -- (8) -- (7); 

\draw[myEdge, dashed, myTip-myTip] (3) -- (7); 
\draw[myEdge, dashed, myTip-myTip] (1) -- (7); 

\node[below left] at (1){\scriptsize 1};
\node[below] at (2){\scriptsize 2};
\node[below right] at (3){\scriptsize 3};
\node[below right] at (4){\scriptsize 4};
\node[below right] at (5){\scriptsize 5};
\node[below left] at (6){\scriptsize 6};
\node[above] at (7){\scriptsize 7};
\node[below right] at (8){\scriptsize 8};
\node[below left] at (9){\scriptsize 9};

\end{tikzpicture}
    \begin{tikzpicture}[x=.7cm, y=.7cm, step=1, baseline = -1cm] 
\node[myNode](1) at (-1,-1){};
\node[myNode](2) at (-1,0){};
\node[myNode](3) at (-1,1){};
\node[myNode](4) at (0,-1){};
\node[myNode](5) at (0,0){};
\node[myNode](6) at (0,1){};
\node[myNode](7) at (1,-1){};
\node[myNode](8) at (1,0){};
\node[myNode](9) at (1,1){};
\draw[myEdge] (1) -- (2) -- (3);
\draw[myEdge] (4) -- (5) -- (6); 
\draw[myEdge] (7) -- (8) -- (9); 
\draw[myEdge] (1) -- (4) -- (7); 
\draw[myEdge] (2) -- (5) -- (8); 
\draw[myEdge] (3) -- (6) -- (9); 
\end{tikzpicture}
    \caption{The graphs described in Examples~\ref{ex:8dotted},
    \ref{ex:9positive_dim}, \ref{ex:solvable_argument}}
    \label{fig:examples_interesting}
\end{figure}
\section{Conclusions} 
We have studied the problem of solvability of viewing graphs, presenting a
series of new theoretical results that can be applied to determine whether a
graph is solvable. We have also pointed out some open questions (particularly,
the relation between finite solvability and solvability, discussed in
Section~\ref{sec:linear}), and we hope that this paper can lead to further work
on these issues.

Our main focus here was to understand whether the camera-estimation problem
is well-posed, and we did not directly address the task of determining the
configuration computationally.  
Properly recovering a global camera configuration that is consistent with local
measurements is challenging, and is arguably the main obstacle for any
structure-from-motion algorithm. For this reason, we believe that a complete
understanding of the algebraic constraints that characterize the compatibility
of fundamental matrices would be very useful. This is an issue that has not
been considered much in classical multi-view geometry, and is very closely
related to the topic of this paper. We plan to investigate it next.
\paragraph{Acknowledgments.} This work was supported in part by the ERC grant
Video\-World, the Institut Universitaire de France, the Inria-CMU associated
team GAYA, ANR Recap, a collaboration agreement between Inria and NYU, and a
grant from the Simons Foundation \#279151.


\bibliographystyle{splncs04}

\renewcommand{\thesection}{\Alph{section}}%

\begin{appendices}

\counterwithin{lemma}{section}
\counterwithin{proposition}{section}
\counterwithin{theorem}{section}
\counterwithin{example}{section}



This supplementary material contains some technical discussions and examples
that were not included in the main body of the paper.

\vspace{.3cm}

\section{Solvability using algebraic geometry}

It is useful to revisit solvability from a slightly more technical viewpoint.
We let $\varphi_G: (\PP^{11})^n \dashrightarrow (\PP^8)^e$ be the rational map
associating to a set of cameras the fundamental matrices corresponding to the
edges of $G$ (so that $e = |E_G|$).\footnote{We recall that rational map
between projective spaces is a map whose coordinates are homogeneous
polynomial functions. The map is in general well-defined only on an open set
of the domain. Note that the definition of $\varphi_G$ actually requires
fixing an orientation for each edge in $G$, since a fundamental matrix
represents an \emph{ordered} pair of cameras. However, it is easy to see that
the orientations of the edges can be chosen arbitrarily.}  The map $\varphi_G$ is
rational because the coefficients of the fundamental matrix are polynomials in
the entries of the two projection matrices. The (closure of the) image of this
map is an algebraic variety in $(\PP^8)^e$ representing the set of
``compatible'' fundamental matrices that can be assigned to the edges of $G$.
We denote this set with $\mathcal X_G$ (as in the main part of the paper).

The solvability of the graph $G$ can be understood in terms of $\varphi_G$ and
$\mathcal X_G$. We will use the following general property of rational maps.

\begin{proposition} Let $f:X \to Y$ be rational map of projective varieties
over an algebraically closed field, and let $Z$ be the closure of the image of
$f$. Then there exists a (Zariski) open dense neighborhood $U$ of $Z$ such that
the fiber of $f$ over $U$ has always: 1) the same (pure) dimension $\dim X - \dim
Z$, and 2) the same number of connected components.
\end{proposition}
\begin{proof} The first point is~\cite[Theorem 7]{shafarevich1994basic}, the second one
is~\cite[Proposition 9.7.8]{grothendieck1965elements}.
\hfill \qed
\end{proof}

This result states that, {\em over an algebraically closed field}, pre-images of
general points will have the same dimension and the same number of connected
components. In our setting, this implies that if we view $\varphi_G$ as a map
on $(\CC\PP^{11})^n$, there exists an open dense set $U_G$ in its image that
defines its general behavior, and one of the following holds:

\begin{enumerate}
    \item the pre-image of every element in $U_G$ is a unique camera
    configuration,
    \item the pre-image of every element in $U_G$ contains a fixed finite number $r>1$ of (complex) camera configurations,
    \item the pre-image of every element in $U_G$ contains infinitely many
    (complex) camera configurations.
\end{enumerate} 

Because uniqueness over $\CC$ is stronger than uniqueness over $\RR$, we see
that necessarily if 1 occurs, the graph $G$ is solvable, and similarly if either 1 or
2 occur, $G$ is finite solvable. On the other hand, it might be possible
for $G$ to be solvable in case $2$, but this would require that exactly one of
the $r$ complex configurations in the general fiber is real.

The next result essentially shows that $G$ is finite solvable if and only
if we are in situations 1 and 2 (because if a set of fundamental is
compatible with an infinite set of complex configurations, it is also
compatible with an infinite set of real ones).

\begin{proposition} 
Suppose that we are in situation 3 above. Then there is a possibly
smaller open dense subset $U'_G \subset U_G$ such that every real fiber
either contains no real configurations or infinitely many real
configurations. Moreover, the latter case occurs over a Zariski dense
subset of $U_G$.
\end{proposition}
\begin{proof}According to the 2nd Bertini Theorem~\cite[Theorem
2]{shafarevich1994basic}, there is an open subset of the image of $\varphi_G$
over which the fibers are smooth. If these fibers have any real points, then
the implicit function theorem from multivariable calculus implies that
real points occur in dimension equal to the complex dimension.
\hfill \qed
\end{proof}

As a corollary we obtain the following result (which implies Proposition 5 in
the main part of the paper).

\begin{proposition} The graph $G$ is finite solvable if and only if $\mathcal
X_G$ has dimension $11n-15$.
\end{proposition}
\begin{proof} It is sufficient to note that if 1 or 2 occur, the generic fiber of $\varphi_G$ has dimension $15$ in $(\PP^{11})^n$. In situation 3, the fiber has higher dimension.
\hfill \qed
\end{proof}

\vspace{.2cm}

\section{Linear equations for finite solvability}

We recall here that $\mathcal T_G({c_1,\ldots,c_n})$ was defined in the main
part of the paper as the set of $e$-tuples $(\vect g_{\lambda}, \lambda \in
E_G)$ of matrices in $GL(4,\RR)$ which satisfied compatibility equations of the
form
\begin{equation}\label{eq:inv_stab2}
\vect g_{\lambda} \vect g_{\lambda'}^{-1} = \alpha \vect I_4 + \vect c_i \vect
v^T,
\end{equation}
where $\alpha$ is an arbitrary (nonzero) constant and $\vect v$ is an arbitrary
vector. The following result will be useful.

\begin{proposition}\label{prop:M} If $\vect c = (c_0,c_1,c_2,c_3)$ is a non-zero vector, then
a matrix $\vect M = (m_{ij})_{i,j=0,\ldots,3}$ can be written in the form
$\vect M = \alpha \vect I_4 + \vect c \vect v^T$ for some arbitrary $\alpha$
and $\vect v$ if and only if the following linear expressions vanish:
\begin{small}
\begin{equation}\label{eq:linear_eqs}
\begin{array}{c}
m_{31} c_{2} -  m_{21} c_{3} \\
m_{30} c_{2} -  m_{20} c_{3} \\
m_{32} c_{1} -  m_{12} c_{3} \\
m_{30} c_{1} -  m_{10} c_{3} \\
m_{23} c_{1} -  m_{13} c_{2} \\
m_{20} c_{1} -  m_{10} c_{2} \\
m_{32} c_{0} -  m_{02} c_{3} \\
m_{31} c_{0} -  m_{01} c_{3} \\
m_{23} c_{0} -  m_{03} c_{2} \\
m_{21} c_{0} -  m_{01} c_{2} \\
m_{13} c_{0} -  m_{03} c_{1} \\
m_{12} c_{0} -  m_{02} c_{1} \\
m_{22} c_{1} -  m_{33} c_{1} -  m_{12} c_{2} + m_{13} c_{3} \\
m_{21} c_{1} -  m_{11} c_{2} + m_{33} c_{2} -  m_{23} c_{3} \\
m_{30} c_{0} -  m_{32} c_{2} -  m_{00} c_{3} + m_{22} c_{3} \\
m_{22} c_{0} -  m_{33} c_{0} -  m_{02} c_{2} + m_{03} c_{3} \\
m_{20} c_{0} -  m_{00} c_{2} + m_{33} c_{2} -  m_{23} c_{3} \\
m_{31} c_{1} -  m_{32} c_{2} -  m_{11} c_{3} + m_{22} c_{3} \\
m_{11} c_{0} -  m_{33} c_{0} -  m_{01} c_{1} + m_{03} c_{3} \\
m_{10} c_{0} -  m_{00} c_{1} + m_{33} c_{1} -  m_{13} c_{3}.
\end{array}
\end{equation}
\end{small}
\end{proposition}
\begin{proof} The result is easily shown using a computer algebra system.
Inside the ring $\QQ[m_{00},\ldots,m_{33},c_0,\ldots,c_3, v_0,\ldots,v_3,
\alpha]$, we consider the ideal $I$ obtained by eliminating the variables
$v_0,v_1,v_2,v_3$ and $\alpha$ from the coordinates of
$\vect M -
\alpha
\vect I_4 + \vect c \vect v^T$. We can then verify that~\eqref{eq:linear_eqs}
generate an ideal that decomposes into two prime components: one of these is
irrelevant for us (it describes the vanishing of $\vect c$) and the other one
is $I$.
\hfill \qed
\end{proof}

We can now prove Proposition 6 from the main part of the paper.

\vspace{.3cm}

\noindent {\bf Proposition 6.} \,{\em The tangent space of $\mathcal T_G({c_1,\ldots,c_n})$ at
the identity can be represented as the space of $e$-tuples of matrices $(\vect
h_{\lambda}, \, \lambda
\in E_G)$ where each $\vect h_{\lambda}$ is in $\RR^{4 \times 4}$ (not
necessary invertible), and with compatibility conditions of the form
\begin{equation}\label{eq:comp_linear2}
\vect h_{\lambda} - \vect h_{\lambda'} = \alpha \vect I_4 + \vect c_i \vect
v^T,
\end{equation} where $\alpha \in \RR \setminus \{0\}$ and $\vect v \in \RR^4$
are arbitrary, and $\lambda$ and $\lambda'$ share the vertex $i$. }

\begin{proof} 
According to Proposition~\ref{prop:M}, a matrix $\vect M$ can be written in
the form $\vect M = \alpha \vect I_4 + \vect c \vect v^T$ for some $\alpha \in
\RR \setminus \{0\}$ and $\vect v \in \RR^4$ if and only if it satisfies a set
linear equations that depend on $\vect c$. Let us write $L_{\vect c}(\vect M)
= \vect 0$ for these linear conditions (so $L_\vect c$ is a linear map). Note
that necessarily $L_{\vect c}(\vect I_4) = \vect 0$.

A constraint of the form~\eqref{eq:inv_stab2} can now be expressed as $F(\vect
g_{\lambda},\vect g_{\lambda'}) = L_{\vect c}(\vect g_{\lambda}\vect
g_{\lambda'}^{-1}) = \vect 0$. Writing the first order expansion or $F(\vect
g_{\lambda},\vect g_{\lambda'})=\vect 0$ at $(\vect I_4, \vect I_4)$
we obtain ($\approx$ denotes equality up to higher order terms)
\begin{equation}
F(\vect I_4 + \vect h_{\lambda}, \vect I_4 + \vect h_{\lambda'}) \approx
L_{\vect c}((\vect I_4 + \vect h_{\lambda})(\vect I_4 - \vect h_{\lambda'}))
\approx L_{\vect c}(\vect h_{\lambda}-\vect h_{\lambda'}).
\end{equation}
This shows that $L_{\vect c}(\vect h_{\lambda}-\vect h_{\lambda'}) = \vect 0$ is
the tangent space at the identity of each constraint~\eqref{eq:inv_stab2}.
By Proposition~\ref{prop:M} we have that $L_{\vect c}(\vect
h_{\lambda}-\vect h_{\lambda'}) = \vect 0$ is equivalent
to~\eqref{eq:comp_linear2}, and this concludes the proof.
\hfill \qed
\end{proof}

Finally, we note that by substituting $\vect M = \vect h_{\lambda} - \vect
h_{\lambda'}$ inside~\eqref{eq:linear_eqs}, we obtain explicit equations that
can be used to determine finite solvability (as explained in Section 4 of the paper).

\vspace{.3cm}

\section{Examples}

\begin{example}\label{ex:four} Let $G$ be the four-cycle shown in
Figure~\ref{fig:four}. We immediately see that the graph is not solvable,
because a solvable graph with four vertices must have at least $e(4)=5$ edges.
It is however useful to understand this example algebraically. Following the
general approach described in Section 4 of the paper, we assign the
identity $\vect I_4$ to the edge $(1,2)$, and unknown matrices $\vect
g_{(2,3)}$, $\vect g_{(3,4)}$, $\vect g_{(1,4)}$ to the remaining edges. The
compatibility equations yield
\begin{equation}\label{eq:comp}
\begin{aligned}
\vect g_{(1,4)} &= \alpha_{1} \vect I_4 + \vect c_1 \vect v_{1}^T\\
\vect g_{(2,3)} &= \alpha_{2} \vect I_4 + \vect c_2 \vect v_{2}^T\\
\vect g_{(3,4)}\vect g_{(2,3)}^{-1} &= \alpha_{3} \vect I_4 + \vect c_3 \vect v_{3}^T\\
\vect g_{(3,4)}\vect g_{(1,4)}^{-1} &= \alpha_{4} \vect I_4 + \vect c_4 \vect v_{4}^T\\
\end{aligned}
\end{equation}
which imply that
\begin{equation}\label{eq:equal}
\vect g_{(3,4)} = (\alpha_{3} \vect I_4 + \vect c_3 \vect v_{3}^T)(\alpha_{2} \vect I_4 + \vect c_2 \vect v_{2}^T)
= (\alpha_{4} \vect I_4 + \vect c_4 \vect v_{4}^T)(\alpha_{1} \vect I_4 + \vect c_1 \vect v_{1}^T).
\end{equation}
Expanding~\eqref{eq:equal} we obtain
\begin{equation}\label{eq:inverse}
(\alpha_3 \alpha_2 - \alpha_4 \alpha_1) \vect I_4 = \vect c_1 \vect w_1^T + \vect c_2 \vect w_2^T + \vect c_3 \vect w_3^T + \vect c_4 \vect w_4^T,
\end{equation}
where 
\begin{equation}
\vect w_1 = \alpha_4 \vect v_{1}, \,\,\, \vect w_2 = - \alpha_3 \vect v_{2}, \,\,\, 
\vect w_3 = -(\alpha_2 \vect I_4 +  \vect v_2 \vect c_{2}^T) \vect v_3, \,\,\, \vect w_4 = (\alpha_1 \vect I_4 +  \vect v_1 \vect c_{1}^T) \vect v_4,
\end{equation}

\begin{figure}[htbp]
    \centering
    \begin{tikzpicture}[x=1.6cm, y=1.6cm, step=1] 
\node[myNode](1) at (0,0){};
\node[myNode](2) at (0,1){};
\node[myNode](3) at (1,1){};
\node[myNode](4) at (1,0){};
\draw[myEdge] (1) -- (2) -- (3) -- (4) -- (1); 

\node[below left] at (1){\small 1};
\node[above left] at (2){\small 4};
\node[above right] at (3){\small 3};
\node[below right] at (4){\small 2};

\node[below] at (.5,0){${\vect I}_4$};
\node[right] at (1,.5){\normalsize ${\vect g}_{(2,3)}$};
\node[above] at (.5,1){\normalsize ${\vect g}_{(3,4)}$};
\node[left] at (0,.5){\normalsize ${\vect g}_{(1,4)}$};
\end{tikzpicture}
    \caption{The four-cycle from Example~\ref{ex:four}}
    \label{fig:four}
\end{figure}
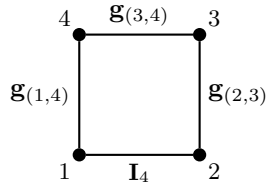

From~\eqref{eq:inverse} we see that the vectors $\vect w_i$ must be scalar
multiples of the rows of the matrix $\vect C^{-1}$ where $\vect C$ has columns
$\vect c_1, \vect c_2,\vect c_3,\vect c_4$. This easily implies that for fixed
general coefficients $\alpha_1,\alpha_2,\alpha_3, \alpha_4$, there is a unique
solution for $\vect v_1,\ldots,\vect v_4$ in~\eqref{eq:comp}. Moreover, since
the matrices $\vect g_{(i,j)}$ represent projective transformations, we can
rescale these equations so that for example $\alpha_1 = \alpha_2 = \alpha_3 =
1$. This shows that there is one degree of projective freedom corresponding to
the choice of $\alpha_4$. This projective freedom can also be explained in
different ways: for example, our analysis implies that $\vect g_{(1,4)}$ can
be any matrix of the form $\vect g_{(1,4)} = \alpha_{1} \vect I_4 + \vect c_1
\vect v_{1}^T$ where $\vect v_1 \cdot \vect c_2 = \vect v_1
\cdot \vect c_3 = \vect v_1 \cdot \vect c_4 = 0$. Not counting the scale
 factor, this gives one degree of freedom, and fixing $\vect g_{(1,4)}$ of this
 type determines all other matrices up to scale.
\hfill $\diamondsuit$
\end{example}


\begin{example}\label{ex:grid} Let $G$ be the graph with $9$ vertices shown in
Figure~\ref{fig:grid}. It was stated in the main part of the paper that this
graph is solvable (Example 6). We can argue this fact using the analysis of
the four-cycle from the previous example. Indeed, if we assign the matrices
$\vect I_4, {\vect g}_{(4,5)}, {\vect g}_{(5,6)}, {\vect g}_{(5,8)}$ as shown
in the figure, then up to rescaling we have that ${\vect g}_{(4,5)} = \vect I_4
+
\vect c_5 \vect v_4^T$ and ${\vect g}_{(5,6)}^{-1} = \vect I_4 +
\vect c_5 \vect v_6^T$, where $\vect v_4$ and $\vect v_6$ are each determined
up to a scalar multiple (more precisely, $\vect v_4$ must satisfy $\vect c_1 \cdot
\vect v_4 =
\vect c_2 \cdot \vect v_4 = \vect c_4 \cdot \vect v_4 = 0$
while $\vect v_6$ must satisfy $\vect c_2 \cdot
\vect v_6 =
\vect c_3 \cdot \vect v_6 = \vect c_6 \cdot \vect v_6 = 0$).\footnote{Note
that the inverse of a matrix of the form $\vect I + \vect c \vect v_1^T$ is
given by $\vect I + \vect c \vect v_2^T$ where $\vect v_2 = - \frac{1}{1+\vect
c^T \vect v_1} \vect v_1$. In particular $\vect v_2$ is a scalar multiple of
$\vect v_1$.} Moreover, we can write
\begin{equation}
\begin{aligned}
\vect g_{(4,5)}\vect g_{(5,8)}^{-1} &=  \vect I_4 + \vect c_5 \vect v_8^T\\
\vect g_{(5,8)}\vect g_{(5,6)}^{-1} &= \alpha \vect I_4 + \vect c_5 \vect v_8'^T\\  
\end{aligned}
\end{equation}
where $\vect c_4 \cdot \vect v_8 = \vect c_7 \cdot \vect v_8 = \vect c_8 \cdot \vect v_8 = 0$
and $\vect c_6 \cdot \vect v_8' = \vect c_8 \cdot \vect v_8' = \vect c_9 \cdot \vect v_8' = 0$.
Multiplying these two expressions together we obtain
\begin{equation}
\vect g_{(4,5)}\vect g_{(5,6)}^{-1}  = (\vect I_4 + \vect c_5 \vect
v_8^T)(\alpha \vect I_4 + \vect c_5 \vect v_8'^T) =(\vect I_4 + \vect c_5
\vect v_4^T)(\vect I_4 + \vect c_5 \vect v_6^T),
\end{equation}
which yields
\begin{equation}
\alpha \vect I_4 + \vect c_5 (\vect v'_8 + (\alpha  + \vect c_5^T \vect v_8')\vect v_8)^T =\vect I_4 + \vect c_5(\vect v_6 + (1 + \vect c_5^T \vect v_6)\vect v_4).
\end{equation}

\begin{figure}[htbp]
    \centering
    \begin{tikzpicture}[x=1.6cm, y=1.6cm, step=1] 
\node[myNode](1) at (-1,-1){};
\node[myNode](2) at (-1,0){};
\node[myNode](3) at (-1,1){};
\node[myNode](4) at (0,-1){};
\node[myNode](5) at (0,0){};
\node[myNode](6) at (0,1){};
\node[myNode](7) at (1,-1){};
\node[myNode](8) at (1,0){};
\node[myNode](9) at (1,1){};
\draw[myEdge] (1) -- (2) -- (3);
\draw[myEdge] (4) -- (5) -- (6); 
\draw[myEdge] (7) -- (8) -- (9); 
\draw[myEdge] (1) -- (4) -- (7); 
\draw[myEdge] (2) -- (5) -- (8); 
\draw[myEdge] (3) -- (6) -- (9); 

\node[below left] at (-1,-1){1};
\node[below left] at (0,-1){2};
\node[below left] at (1,-1){3};
\node[below left] at (-1,0){4};
\node[below left] at (0,0){5};
\node[below left] at (1,0){6};
\node[below left] at (-1,1){7};
\node[below left] at (0,1){8};
\node[below left] at (1,1){9};

\node[right] at (0,-.5){${\vect I}_4$};
\node[above] at (.5,0){${\vect g}_{(5,6)}$};
\node[above] at (-.5,0){${\vect g}_{(4,5)}$};
\node[right] at (0,.5){${\vect g}_{(5,8)}$};
\end{tikzpicture}
    \caption{The graph from Example~\ref{ex:grid}}
    \label{fig:grid}
\end{figure}
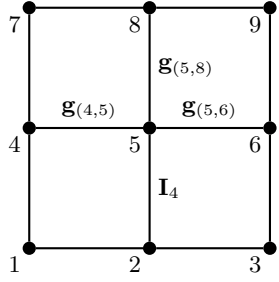

This relation can only be satisfied if $\alpha = 1$ and 
\begin{equation}\label{eq:condition}
\vect v'_8 + (1  + \vect c_5^T \vect v_8')\vect v_8 = \vect v_6 + (1 + \vect c_5^T \vect v_6)\vect v_4.
\end{equation}
Because the orthogonality conditions that defined $\vect v_4,\vect v_6, \vect v_8, \vect v_8'$ are all independent
(since the pinholes are generic), \eqref{eq:condition} can be satisfied only if these vectors are all zero. From this we deduce that all the matrices $\vect g_{(i,j)}$ must be the identity, and the graph $G$ is solvable.
\hfill $\diamondsuit$
\end{example}

\vspace{.5cm}

\begin{example}\label{ex:dotted} Let $G$ be the graph shown in
Figure~\ref{fig:dotted}. As argued in the main part of the paper (Example 4),
we can show that this graph is solvable by proving that the image of the
pinhole 1 in the view 7 is fixed by the structure of the graph. Indeed,
we note that:
\begin{enumerate}
    \item The image of the pinhole 1 in the view 5 lies on the line $l_{51}$ defined
    by reprojection of the epipole $e_{41}$ from the view 4 to the view 5.
    \item Similarly, the image of the pinhole 1 in the view 6 lies on the line $l_{61}$
    defined by the reprojection of the epipole $e_{31}$ from the view 3 to the
    view 6.
    \item Using the previous two observations, we deduce that the image of the pinhole 1
    in the view 7 lies on the line $l_{71}$ obtained by transferring
    $l_{51}$ and $l_{61}$ to the view $7$. This line is the projection from 7
    of the intersection of the planes spanned by $c_1,c_4,c_5$ and
    $c_1,c_3,c_6$.
    \item The image of the pinhole 1 in the view 7 belongs to another line $m_{71}$ that
    is the reprojection of the epipole $e_{81}$ from the view 8 to the
    view~7.
    \item For generic pinholes the lines $m_{71}$ and $l_{71}$ will be different, so
    their intersection determines the image of the pinhole 1 in the view 7.
\end{enumerate}
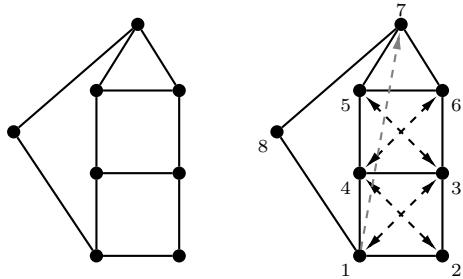
\begin{figure}[htbp]
    \centering
    \begin{tikzpicture}[x=1.1cm, y=1.1cm, step=1] 
\node[myNode](1) at (0,0){};
\node[myNode](2) at (1,0){};
\node[myNode](3) at (1,1){};
\node[myNode](4) at (0,1){};
\node[myNode](5) at (0,2){};
\node[myNode](6) at (1,2){};
\node[myNode](7) at (0.5,2.8){};
\node[myNode](8) at (-1,1.5){};
\draw[myEdge] (1) -- (2) -- (3) -- (4) -- (1);
\draw[myEdge] (4) -- (5) -- (6) -- (3); 
\draw[myEdge] (5) -- (7) -- (6); 
\draw[myEdge] (7) -- (8) -- (1);

\node[shift={(3.5 cm, 0 cm)}, myNode](1) at (0,0){};
\node[shift={(3.5 cm, 0 cm)}, myNode](2) at (1,0){};
\node[shift={(3.5 cm, 0 cm)}, myNode](3) at (1,1){};
\node[shift={(3.5 cm, 0 cm)}, myNode](4) at (0,1){};
\node[shift={(3.5 cm, 0 cm)}, myNode](5) at (0,2){};
\node[shift={(3.5 cm, 0 cm)}, myNode](6) at (1,2){};
\node[shift={(3.5 cm, 0 cm)}, myNode](7) at (0.5,2.8){};
\node[shift={(3.5 cm, 0 cm)}, myNode](8) at (-1,1.5){};
\draw[myEdge] (1) -- (2) -- (3) -- (4) -- (1);
\draw[myEdge] (4) -- (5) -- (6) -- (3); 
\draw[myEdge] (5) -- (7) -- (6); 
\draw[myEdge] (7) -- (8) -- (1);
\draw[myEdge,dashed, myTip-myTip] (1) -- (3);
\draw[myEdge,dashed, myTip-myTip] (2) -- (4);
\draw[myEdge,dashed, myTip-myTip] (3) -- (5);
\draw[myEdge,dashed, myTip-myTip] (4) -- (6);
\draw[myEdge,dashed,-myTip,gray] (1) -- (7);

\node[below left] at (1){\scriptsize 1};
\node[below right] at (2){\scriptsize 2};
\node[below right] at (3){\scriptsize 3};
\node[below left] at (4){\scriptsize 4};
\node[below left] at (5){\scriptsize 5};
\node[below right] at (6){\scriptsize 6};
\node[above] at (7){\scriptsize 7};
\node[below left] at (8){\scriptsize 8};
\end{tikzpicture}
    \caption{The graph from Example~\ref{ex:dotted}}
    \label{fig:dotted}
\end{figure}

This argument shows that we can draw a dashed arrow from the vertex $1$ to the
vertex $7$. We can now prove that $G$ is solvable using the following sequence of
moves (which starts from the situation illustrated at the right of Figure~\ref{fig:dotted}).
\begin{itemize}
    \item Draw double dashed arrows (4,7) and (3,7) (move II).
    \item Draw a dashed arrow from 1 to 5 (move II for the four-cycle (1,4,5,7)).
    \item Make (3,5) solid (move III), and then also (4,6), and (4,7) and (3,7) (move I).
    \item Make (1,7) double dashed (move II for (1,4,7,8)).
    \item Make (1,3) solid (move III) and complete the graph using move I.
\end{itemize}
    \hfill $\diamondsuit$
\end{example}

\end{appendices}

\end{document}